\documentclass[10pt,journal]{IEEEtran}
\IEEEoverridecommandlockouts

\usepackage{amsmath,amssymb,amsthm}
\usepackage{graphicx}
\usepackage{booktabs}
\usepackage{algorithm}
\usepackage{algpseudocode}
\usepackage{multirow}
\usepackage{hyperref}
\usepackage{xcolor}

\newtheorem{proposition}{Proposition}
\newtheorem{assumption}{Assumption}
\newtheorem{definition}{Definition}
\newtheorem{lemma}{Lemma}

\newcommand{\R}{\mathbb{R}}
\newcommand{\E}{\mathbb{E}}
\newcommand{\Prb}{\mathbb{P}}
\newcommand{\bc}{\bar{c}}
\newcommand{\sig}{\sigma}
\newcommand{\Sset}{\mathcal{S}}
\newcommand{\Aset}{\mathcal{A}}
\newcommand{\Pset}{\Pi_{\sig}}

\newif\ifanonymous
\anonymousfalse  

\title{Dijkstra as an Oracle for Online Stochastic Shortest Path
Navigation with Provable Guarantees}

\ifanonymous
\author{Author names omitted for double-anonymous review}
\else
\author{Mansur Arief$^{1}$, Ali Akarma$^{2}$, Ahmad Alfan Alfian Irfan$^{3}$
\thanks{$^{1}$Industrial and Systems Engineering Department, IRC Smart Mobility and Logistics, King Fahd University of Petroleum and Minerals (KFUPM), Saudi Arabia. Email: {\tt\small mansur.arief@kfupm.edu.sa}}
\thanks{$^{2}$Faculty of Computer and Information Systems, Islamic University of Madinah, Saudi Arabia. Email: {\tt\small 443059463@stu.iu.edu.sa}}
\thanks{$^{3}$Information Technology, Universitas Muhammadiyah Yogyakarta, Indonesia. Email: {\tt\small ahmad.alfan.ft23@mail.umy.ac.id}}
\thanks{$^{\dagger}$This work has been submitted to the IEEE for possible publication. Copyright may be transferred without notice, after which this version may no longer be accessible.}
\thanks{$^{*}$The code is available online at \url{https://github.com/ai-vnv/DORASolvers.jl}}
}
\fi

\begin{document}
\maketitle
\thispagestyle{empty}
\pagestyle{empty}

\begin{abstract}
Mobile robots that operate in side by side with humans and critical facilities must reach
their goals at low cost, despite often unknown true traversal costs of the map apriori and imperfect actuation. Planners that solve the underlying stochastic
shortest path problem exactly, such as value iteration, require computation
that grows with the diameter of the map, whereas Dijkstra's algorithm is fast
but is usually considered inexact once transitions are stochastic. This
study shows that Dijkstra's algorithm can remain an exact planning engine under
a condition that is much weaker than the causality condition often invoked in the
literature, namely nonnegativity of a reduced cost defined on
the determinized map. Building on this characterization, 
an online learner DORA (Dijkstra Oracle Reduced-cost Algorithm) is proposed for robot navigation that calls a shortest path oracle a fixed number of times per episode, never estimates a transition kernel, and
adds a logarithmic survival weight when the probability of contact with a
dynamic obstacle must stay within a budget. In the numerical experiments involving
three other benchmarks that cover grid world navigation, directional
drilling, and drone surveillance, the learner matches optimistic value
iteration that is given the true transition kernel while performing 4.5 to
19.3 times less planner work, reduces contacts during learning by a factor of
seventeen relative to determinize and replan, and keeps the contact rate
within budgets that span two orders of magnitude. These results indicate that
shortest path search supports safe and efficient online navigation and path planning tasks.

\end{abstract}


\section{Introduction}

Autonomous and service robots often deal with problems involving known maps and goals, and the robot must reach the goal at low cost while
avoiding contact with obstacles and people \cite{marder2010,trautman2010,malathi2025}. Under such a setting, what is not known in advance is how
expensive each part of the map really is. Floor condition, local clutter and
foot traffic all change the effective cost of a route, and they are learned
in real-time only by traversing. The robot must therefore improve its routing policy online
while it continues to operate.

In the literature, we often cast this setting as a stochastic shortest path problem, or SSP \cite{bertsekas1991}.
Key characteristics include actuation being imperfect, so a commanded move sometimes slips and contact with a
dynamic obstacle ends the task. Value iteration and its heuristic search
variants solve such problems to optimality \cite{hansen2001,bonet2003lrtdp},
and recent work has established minimax regret rates for learning them online
\cite{rosenberg2020,cohen2021,tarbouriech2021}. However, the difficulty is usually
computational. Value iteration propagates information one step per sweep, so
the number of sweeps needed grows with the diameter of the graph. A robot that
replans at every cycle incurs this compute cost every cycle.

Meanwhile, Dijkstra's algorithm \cite{dijkstra1959} has the opposite profile. A single pass propagates cost
information across the entire map, and the run time is near linear in the
number of edges. It is also the workhorse of deployed navigation stacks \cite{marder2010,macenski2020}. The
usual objection is that it does not apply to stochastic problems. The classical
condition for a one pass label setting method to be exact is that an optimal
policy is consistently improving, meaning the value strictly decreases along
every transition that has positive probability \cite{vladimirsky2008,
gaspard2023}. As soon as an actuator slips sideways into a more expensive
region, that condition fails.

As shown in later sections, our analysis shows that this objection is stated too strongly. The minimal condition is nonnegativity of a reduced cost. To see this, for each state and action, define
the reduced cost as the state action value of the stochastic problem minus the
value of the state the action was intended to reach. If these quantities are
nonnegative, then running Dijkstra with them as edge weights returns exactly
the optimal value function and exactly the optimal policy. That said, the reduced cost is
nonnegative under a much weaker requirement than causality, because it only
constrains the determinized edge, not the slip outcomes. In our
study the classical causality condition holds at about half of the states with the
slip present, while the reduced costs remain non-negative on state-action pairs up to a slip probability of $0.30$.

This observation suggests an algorithm. The reduced costs are unknown, but they
have a self consistent form: the reduced cost is the learned step cost plus the
expected change in cost to goal caused by slip. We therefore iterate. Given a
current estimate of the cost to goal, we form the reduced cost weights, call
Dijkstra, and use the returned labels to update the estimate. A small number of
damped iterations is enough. The result is the proposed Dijkstra Oracle Reduced-cost Algorithm (DORA). DORA maintains optimistic estimates of the $|\Sset||\Aset|$
traversal costs and never estimates a transition kernel, which would require
$|\Sset|^2|\Aset|$ parameters.

Our contributions are threefold. First, we characterize when the policy class induced by Dijkstra contains an
optimal policy for a stochastic shortest path problem. The condition is
nonnegativity of a reduced cost on the determinized map, and it is strictly
weaker than the causality condition used in the label setting literature.
Second, we introduce DORA, an online learner that reaches this class using a
fixed number of shortest path oracle calls per episode and no transition
model, together with a chance constrained variant that enforces a risk budget
through an additive log survival weight.
Third, we evaluate the method on a warehouse navigation benchmark. DORA matches optimistic value iteration supplied
with the true kernel while doing an order of magnitude less planner work, and
it reduces contacts during learning by a factor of seventeen relative to plain
determinize and replan. The same advantages hold on three further navigation
benchmarks.

The remainder of the paper is organized as follows. Section~\ref{sec:related}
reviews the related work. Section~\ref{sec:problem} states the problem.
Section~\ref{sec:method} develops the reduced cost characterization and the
algorithm. Section~\ref{sec:exp} reports the experiments.
Section~\ref{sec:disc} discusses the findings, and Section~\ref{sec:concl}
concludes.

\section{Related Work}\label{sec:related}

\subsection{Shortest Path Search for Robot Navigation}

Graph search is the standard planning layer in mobile robotics. D* Lite reuses
information across a sequence of searches and returns the same optimal path
that A* \cite{hart1968} would return on the current cost map \cite{koenig2002}. ARA* attaches
an explicit suboptimality factor to every anytime solution \cite{likhachev2003},
and Anytime D* combines the two while preserving the bound
\cite{likhachev2005}. These methods assume deterministic transitions. When
transitions are stochastic the solution is a policy rather than a path, and
LAO* extends heuristic search to solution graphs with loops
\cite{hansen2001}. Labeled real-time dynamic programming adds a convergence
bound that the plain counterpart is lacking \cite{bonet2003lrtdp}.

The idea of using Dijkstra on a stochastic problem is not new. McMahan and
Gordon generalize Dijkstra and Gaussian elimination into a family of exact MDP
solvers that reduce to Dijkstra when the transitions are deterministic
\cite{mcmahan2005gen}. Bounded real time dynamic programming uses a Dijkstra
sweep to build a monotone initial bound \cite{mcmahan2005brtdp}. Topological
value iteration performs label setting at the granularity of strongly connected
components \cite{dai2007,dai2011}. Bertsekas develops a Dijkstra like algorithm
for robust shortest paths under a semicontractive model \cite{bertsekas2019}.
The exactness conditions for one pass label setting on stochastic problems were
established by Vladimirsky, who requires a consistently improving optimal
policy \cite{vladimirsky2008}, and sharpened by Gaspard and Vladimirsky, who
give explicit cost conditions and an explicit counterexample when the condition
is violated \cite{gaspard2023}. Our contribution is to show that a weaker
condition, stated on reduced costs, rather than on raw values, is what governs
whether Dijkstra-derived policies are lossless.

\subsection{Online Learning for Goal Oriented Problems}

Regret guarantees for stochastic shortest path learning have advanced quickly.
UC-SSP gave the first no regret algorithm without restrictive assumptions
\cite{tarbouriech2020}, near optimal bounds in terms of the optimal cost to go
followed \cite{rosenberg2020}, and the minimax rate was settled by two
concurrent works \cite{cohen2021,tarbouriech2021}. Policy optimization for this
setting was initiated by Chen, Luo and Rosenberg \cite{chen2022po}, and Chen et
al. \cite{chen2023} proved that horizon free regret is impossible under general costs. All of these algorithms plan by value iteration or by
optimization over occupancy measures, and none treats a combinatorial planner
as a black box oracle. Oracle-based learners are standard in online
combinatorial optimization, where following the perturbed leader calls an offline
solver on perturbed costs \cite{kalai2005}, but that line assumes deterministic
edges. We place a shortest path oracle inside a learner for a stochastic
shortest path problem.

\subsection{Determinization and the Price of Structure}

Replanning with a determinized model is a strong practical baseline in
probabilistic planning \cite{yoon2007}. Little and Thi\'ebaux \cite{little2007} showed by
construction that the policy induced by the most likely trajectory can be
exponentially worse than optimal, and later work on reduced
models states plainly that little can be guaranteed about the quality of plans
produced by replanners \cite{pineda2019}. Hansen \cite{hansen2011} gives an a posteriori
suboptimality bound for stochastic shortest path in which the expected time to
absorption replaces the discount amplifier. Our experiments
show that this bound is vacuous at realistic slip levels, while the actual gap
of the reduced cost member of the class is zero. The picture that emerges is
that determinization is not intrinsically lossy. What is lossy is weighting the
determinized graph by expected one step cost.

\subsection{Safe Navigation under Uncertainty}

Chance constrained planning allocates a risk budget across constraints and
solves a tightened deterministic problem \cite{ono2008}. RAO* searches belief
states with admissible bounds on both utility and execution risk and returns
optimal policies that satisfy the chance constraint \cite{santana2016}. FIRM
makes edge costs independent in belief space so that the optimal substructure
needed by graph search is restored \cite{agha2014}. In probabilistic
verification, weighting each edge by the negative logarithm of its probability
turns a most probable violating path into a shortest path \cite{han2009}. Our
chance constrained variant uses the same transformation on the planning side.
Related work also shows that a deterministic solver used as a black box
suffices for several risk objectives with only a logarithmic number of calls
\cite{nikolova2010}, and that a budget of uncertainty costs at most $n+1$
nominal solves \cite{bertsimas2003}.

\section{Problem Formulation}\label{sec:problem}

We model navigation as a stochastic shortest path problem
$M = (\Sset, \Aset, P, c, s_0, g)$. The state set $\Sset$ contains the free
cells of a known occupancy map together with an absorbing crash state, and
$\Aset$ contains four motion primitives. The robot starts at $s_0$ and must
reach the goal $g$. Transitions are stochastic because actuation is imperfect.
A commanded action $a$ moves the robot to its intended successor with
probability $1-\varepsilon$ and slips to one of the two lateral cells with
probability $\varepsilon/2$ each. A slip into a shelf stops the robot in place
and incurs a bump cost. Entering a cell that is occupied by a dynamic obstacle
sends the robot to the crash state, which carries a known dead end penalty.

The traversal cost $c(s,a)$ is unknown. At each step the robot observes a noisy
sample of it. Costs are bounded below by $c_{\min} > 0$ and above by
$c_{\max}$. We write $\bc(s,a)$ for the true mean cost. The value of a
stationary policy $\pi$ is the expected cost accumulated before absorption,
truncated at a horizon $H$ with a timeout penalty, and $V^\star$ is the
optimal value. We write 
\begin{equation}
    Q^\star(s,a) = \bc(s,a) + \sum_{s'} P(s'|s,a)
V^\star(s').
\end{equation}

\noindent The map geometry gives a determinized graph. Let $\sig(s,a)$ denote the cell
that action $a$ is intended to reach from $s$, and let $E$ be the set of pairs
for which this is defined. We make the following standing assumptions, which
hold for any connected occupancy map.

\begin{assumption}\label{as:reach}
The goal $g$ is reachable from every state along determinized edges, and
$c(s,a) \in [c_{\min}, c_{\max}]$ with $c_{\min} > 0$.
\end{assumption}

\noindent The learner interacts for $K$ episodes. In episode $k$ it commits to a
stationary policy $\pi_k$, executes it until absorption or timeout, and updates
its estimates. We measure performance by the regret
\begin{equation}
R_K = \sum_{k=1}^{K} \left( J^{\pi_k}(s_0) - V^\star(s_0) \right),
\label{eq:regret}
\end{equation}
by the success, contact and timeout rates of the executed policy, and by the
work performed by the planner per episode. Work is counted in elementary
operations, namely edge scans for a shortest path call and successor
evaluations for a value iteration sweep. We report work rather than wall clock
so that the comparison does not depend on how each planner is implemented.

\section{Method}\label{sec:method}

\subsection{The Dijkstra Policy Class}

Any nonnegative weight vector on the determinized graph induces a policy.

\begin{definition}\label{def:class}
For $w \in \R_{\ge 0}^{E}$ let $d_w$ solve
$d_w(s) = \min_{a} \{ w(s,a) + d_w(\sig(s,a)) \}$ with $d_w(g) = 0$, and let
$\pi_w(s) = \arg\min_a \{ w(s,a) + d_w(\sig(s,a)) \}$. The Dijkstra policy
class is $\Pset = \{ \pi_w : w \in \R_{\ge 0}^{E} \}$.
\end{definition}

\noindent Under Assumption~\ref{as:reach} the labels $d_w$ are the unique solution of
that system and Dijkstra computes them in $O(|E| + |\Sset| \log |\Sset|)$ time.
Note that a single backward call from the goal returns a policy
defined at every state, not a single path, so the robot may be displaced by a
slip and still act greedily without replanning. Furthermore, the class is finite with
cardinality at most $|\Aset|^{|\Sset|}$, so it cannot be searched by
enumeration. The oracle searches it implicitly.

\subsection{Reduced Costs and Exactness}

The natural choice $w = \bc$ is the one often used by determinize and replan
planners \cite{yoon2007}. However, the right choice is a reduced cost defined in Definition~\ref{def:rc} and shown to be exact in Proposition~\ref{prop:exact}.

\begin{definition}\label{def:rc}
The reduced cost of $(s,a) \in E$ is
\begin{equation}
w^\star(s,a) = Q^\star(s,a) - V^\star(\sig(s,a)).
\label{eq:rc}
\end{equation}
\end{definition}

\begin{proposition}[Exact representation]\label{prop:exact}
Suppose Assumption~\ref{as:reach} holds and $w^\star(s,a) \ge 0$ for all
$(s,a) \in E$. Then $d_{w^\star} = V^\star$ and $\pi_{w^\star}$ is an optimal
policy for $M$. In particular, $\Pset$ contains an optimal policy and
Dijkstra recovers it in one call.
\end{proposition}

\begin{proof}
By the Bellman equation and \eqref{eq:rc},
\[
\min_a \{ w^\star(s,a) + V^\star(\sig(s,a)) \} = \min_a Q^\star(s,a)
= V^\star(s),
\]
and $V^\star(g) = 0$. So $V^\star$ solves the determinized fixed point system
of Definition~\ref{def:class} with weights $w^\star$. Under
Assumption~\ref{as:reach} and $w^\star \ge 0$ that system has a unique
solution, hence $d_{w^\star} = V^\star$. The greedy action then satisfies
$\arg\min_a \{ w^\star(s,a) + d_{w^\star}(\sig(s,a)) \} =
\arg\min_a Q^\star(s,a)$, which is optimal.
\end{proof}

The condition in Proposition~\ref{prop:exact} is weaker than the condition used
in the label-setting literature \cite{vladimirsky2008,gaspard2023}. A policy is consistently improving when
$V^\star(s) > V^\star(s')$ for every $s'$ with $P(s'|s,\pi^\star(s)) > 0$,
which constrains all slip outcomes \cite{vladimirsky2008,gaspard2023}. We thus have Proposition~\ref{prop:caus}.

\begin{proposition}[Relation to causality]\label{prop:caus}
At the optimal action, $w^\star(s,\pi^\star(s)) = V^\star(s) -
V^\star(\sig(s,\pi^\star(s)))$. Hence nonnegativity of the reduced cost at the
optimal action is exactly the causality condition restricted to the
determinized edge, and it is implied by, but does not imply, consistent
improvement.
\end{proposition}

\begin{proof}
$Q^\star(s,\pi^\star(s)) = V^\star(s)$ by optimality, and substitution into
\eqref{eq:rc} gives the identity. Consistent improvement requires a strict
decrease on all positive probability successors, of which the determinized
successor is one, so it implies the stated inequality. The converse fails
because a lateral slip may raise the value without affecting $w^\star$.
\end{proof}

This distinction explains the empirical picture reported in
Section~\ref{sec:exp}. Slip breaks causality immediately, because some lateral
outcome always lands in a more expensive region. It does not break
nonnegativity of the reduced cost, because a step still costs at least
$c_{\min}$ and still makes progress toward the goal in the determinized graph.

\subsection{The DORA Algorithm}

The reduced costs are unknown, but \eqref{eq:rc} can be written in a self
consistent form. Substituting the definition of $Q^\star$,
\begin{equation}
w^\star(s,a) = \bc(s,a) + \underbrace{\sum_{s'} P(s'|s,a) V^\star(s')
- V^\star(\sig(s,a))}_{\text{drift}} .
\label{eq:drift}
\end{equation}
The drift term is the expected penalty for not landing where the action
intended. It requires the actuation model, which is a single calibrated scalar
$\varepsilon$ together with the known map geometry, and it requires an estimate
of the cost to goal, which the oracle itself returns.

DORA therefore alternates between forming the weights and calling the oracle.
Let $\hat{c}_k$ be the empirical mean cost,
\begin{equation}
b_k(s,a) = \beta \sqrt{ \frac{\log(2|\Sset||\Aset|K/\delta)}{\max(N_k(s,a),1)} }
\label{eq:bonus}
\end{equation}
be a confidence radius, and 
\begin{equation}
    w^{\mathrm{base}}_k = \max(c_{\min}, \hat{c}_k - b_k)
\end{equation} 
be the optimistic step cost. Given a current label vector
$d$, the algorithm forms $\left(w(s,a)\right)^+$ where 
\begin{equation}
w(s,a) = w^{\mathrm{base}}_k(s,a) + \sum_{s'} \hat{P}(s'|s,a) d(s')
- d(\sig(s,a)),
\label{eq:w}
\end{equation}
calls Dijkstra, and updates $d$ by a damped step with parameter $\alpha$. Note here that the
clipping at zero is required because the oracle needs nonnegative weights, and
it is inactive whenever Proposition~\ref{prop:exact} applies. The full
procedure is given in Algorithm~\ref{alg:dora}. At this stage, we have:

\begin{algorithm}[t]
\caption{DORA}
\label{alg:dora}
\begin{algorithmic}[1]
\Require map $\sig$, actuation model $\hat P$, episodes $K$, inner iterations
$I$, damping $\alpha$, radius scale $\beta$, dead end penalty $c_{\mathrm{de}}$
\State $N \gets 0$, $C \gets 0$, $d \gets 0$, $d(\text{crash}) \gets
c_{\mathrm{de}}$
\State build reverse adjacency of $\sig$ once
\For{$k = 1, \dots, K$}
  \State $w^{\mathrm{base}} \gets \max(c_{\min}, C/N - b_k)$
    \Comment{Eq.~\eqref{eq:bonus}}
  \For{$i = 1, \dots, I$}
    \State form $w$ from $w^{\mathrm{base}}$ and $d$
      \Comment{Eq.~\eqref{eq:w}}
    \State $(\pi_k, d^{+}) \gets \textsc{Dijkstra}(\sig, w, g)$
    \State $d \gets (1-\alpha) d + \alpha d^{+}$;\;
      $d(\text{crash}) \gets c_{\mathrm{de}}$
  \EndFor
  \State execute $\pi_k$ for one episode
  \State update $N$ and $C$ from the observed step costs
\EndFor
\State \Return $\pi_K$
\end{algorithmic}
\end{algorithm}

\begin{proposition}[Fixed point]\label{prop:fp}
Suppose $\hat{c} = \bc$, $\hat{P} = P$, $b \equiv 0$ and $w^\star \ge 0$ on
$E$. Then $d = V^\star$ is a fixed point of the inner loop of
Algorithm~\ref{alg:dora}, and the returned policy is optimal.
\end{proposition}

\begin{proof}
Setting $d = V^\star$ in \eqref{eq:w} and comparing with \eqref{eq:drift} gives
$w = [w^\star]^{+} = w^\star$. Proposition~\ref{prop:exact} then yields
$d^{+} = d_{w^\star} = V^\star$, so the update leaves $d$ unchanged and
$\pi_k = \pi_{w^\star}$ is optimal.
\end{proof}

Proposition~\ref{prop:fp} states what the planner converges to when the
estimates are correct. However, it does not by itself give a regret bound. Furthermore, because $w^{\mathrm{base}}_k$ is an optimistic estimate
of $\bc$ on the event that all confidence intervals hold, the policy returned
at episode $k$ is greedy with respect to an optimistic model, and the per
episode regret is controlled by the sum of confidence radii along the visited
trajectory. Following the argument used for optimistic algorithms in this
setting \cite{tarbouriech2021,rosenberg2020}, we have a bound of order
$\tilde{O}( c_{\max} \tau_{\max} \sqrt{|\Sset||\Aset|K} )$ against the best
member of $\Pset$, where $\tau_{\max}$ bounds the expected time to absorption.
Even further, under Proposition~\ref{prop:exact} the best member of $\Pset$ is optimal, so
the same bound holds against $V^\star$. We state this as the guarantee that
motivates the design and verify the resulting sublinear behavior empirically in
Section~\ref{sec:exp}. The formal guarantee for the variant whose inner iteration is run to convergence is provided in Appendix~\ref{app:proofs}.

In terms of compute, the cost of Algorithm~\ref{alg:dora} is $I$ oracle calls and $I$ drift
evaluations per episode. The advantage is the number $I$ does not grow with the map. A single
Dijkstra pass already propagates cost information across the whole graph, so
the iteration only has to correct for slip (which is local). On the contrary, value iteration
has no such property. Its number of sweeps grows with the diameter of the
graph, because each sweep moves information one step. This difference is the source of the
scaling behavior reported in Section~\ref{sec:exp1}.

\subsection{Chance-Constrained Navigation}

Finally we address safety, which is often specified as a budget \cite{ono2008,blackmore2011,degroot2025}. Let
$\rho(\pi)$ be the probability that an episode ends in contact, and require
$\rho(\pi) \le \Delta$. Let $\hat{p}(s,a)$ be the empirical contact
probability. For that, we add a log survival term to the oracle weight,
\begin{equation}
w_\lambda(s,a) = w(s,a) + \lambda \left( -\log \left(1 - \hat{p}(s,a)\right)
\right).
\label{eq:risk}
\end{equation}
The weight of a route under the second term alone is the negative logarithm of
the probability of traversing it without contact, so a shortest path under
\eqref{eq:risk} trades distance off against survival. The same transformation
underlies counterexample generation in probabilistic verification
\cite{han2009}. Furthermore, the multiplier is updated by projected dual ascent on the
realized contact indicator of the episode just executed,
\begin{equation}
\lambda_{k+1} = \Big[ \lambda_k + \tfrac{\eta_0}{\sqrt{k}}
\left( \mathbf{1}\{\text{contact}_k\} - \Delta \right) \Big]_{0}^{\lambda_{\max}},
\label{eq:dual}
\end{equation}
which needs no extra planning sweep. Here, $\mathbf{1}\{\cdot\}$ is the indicator function. We call this variant DORA-S. Because the
optimal policy of a constrained Markov decision process is in general a mixture
of deterministic policies \cite{altman1999}, the constraint is met in the time
averaged sense rather than by every individual iterate.

\begin{figure*}[htbp]
\centering
\includegraphics[width=0.72\textwidth]{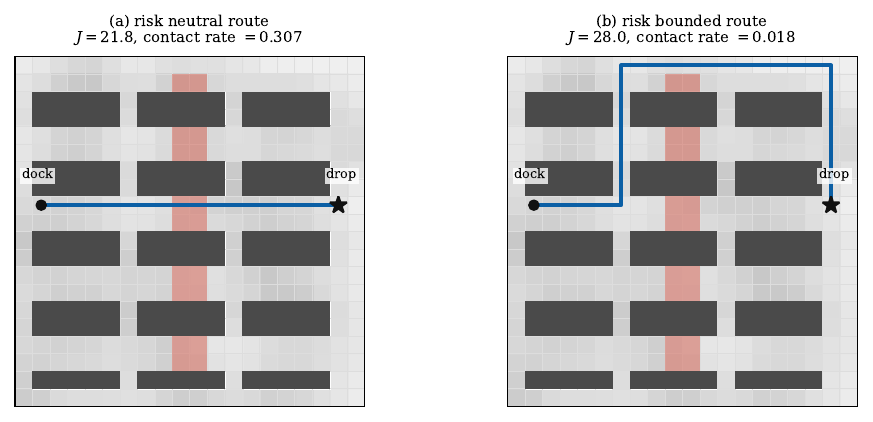}
\caption{The warehouse benchmark. Dark blocks are shelves, the shading of free
cells is the unknown traversal cost, and the red band is a picking zone where
contact ends the episode. (a) The risk neutral route drives through the
picking zone. (b) With a risk budget the same oracle call returns a detour
that almost eliminates contact.}
\label{fig:map}
\end{figure*}

\section{Experiments}\label{sec:exp}

\subsection{Setup}

The benchmark we provide is a $20 \times 20$ warehouse with shelf blocks and four vertical
cross aisles, giving $266$ states after adding the crash state. A human picking
zone spans two interior columns. The robot may drive through it or detour
through the perimeter aisles, which are the only gates left open.
Figure~\ref{fig:map} shows the map and the two routes. Traversal costs combine
a shelf proximity term with a smoothed random field, so that the region a robot
slips into affects its performance and not only the step on which it slipped. Unless stated
otherwise the slip probability is $\varepsilon = 0.10$, the contact probability
in the picking zone is $0.14$, the dead end penalty is $80$, the horizon is
$150$ and the radius scale is $\beta = 0.05$. Every configuration is repeated
over $10$ instances that differ in the random terrain field, and we report the
mean and the standard deviation of the regret, along with the other key metrics, which include Work, SR, CR, and Contacts. Here, work is planner operations per episode. SR and CR are the success and contact
rates of the final policy, respectively. Finally, contacts is the expected number accumulated during learning.

The model, the Dijkstra implementation and all learners are written in Julia.
The oracle is an indexed binary heap Dijkstra with decrease key. We
evaluate each distinct policy exactly by backward recursion (rather than by
sampling), which removes evaluation noise from every reported curve.

We compare the following methods. DORA is Algorithm~\ref{alg:dora} with $I=3$
and $\alpha = 0.4$. DORA-0 is the ablation with the reduced cost correction
switched off, which is optimistic determinize and replan. CED \cite{yoon2007} is the
certainty equivalent Dijkstra planner with neither optimism nor correction. EGD
adds $\varepsilon$ greedy exploration to CED. OVI-U is optimistic value
iteration that estimates the transition kernel \cite{tarbouriech2020,rosenberg2020}, and OVI-K uses the true kernel. Both value iteration baselines are run
to convergence with a warm start at every episode, so they are not penalized by
an arbitrary sweep budget.

\subsection{Online Navigation}\label{sec:exp1}

Table~\ref{tab:exp1} reports the main comparison over
$K = 800$ episodes. We summarize the results also in Figure~\ref{fig:exp1}. Our findings are as follows.

\begin{table}[t]
\centering
\caption{Online navigation metric comparison over $800$ episodes}
\label{tab:exp1}
\footnotesize
\setlength{\tabcolsep}{3.2pt}
\begin{tabular}{lrrrrr}
\toprule
Method & Regret $\downarrow$ & Work $\downarrow$ & SR $\uparrow$ &
CR $\downarrow$ & Contacts $\downarrow$ \\
\midrule
DORA   & $\mathbf{203 \pm 74}$   & $15{,}084$  & $\mathbf{0.984}$ & $\mathbf{0.016}$ & $\mathbf{12.7}$ \\
DORA-0 & $4{,}806 \pm 866$          & $\mathbf{772}$ & $0.727$ & $0.273$ & $219.4$ \\
CED    & $4{,}751 \pm 841$          & $\mathbf{772}$ & $0.727$ & $0.273$ & $218.7$ \\
EGD    & $7{,}753 \pm 740$          & $\mathbf{772}$ & $0.709$ & $0.291$ & $232.6$ \\
OVI-U  & $1{,}906 \pm 1{,}666$         & $196{,}094$ & $0.983$ & $0.017$ & $20.5$ \\
OVI-K  & $203 \pm 47$            & $164{,}801$ & $\mathbf{0.984}$ & $\mathbf{0.016}$ & $\mathbf{12.7}$ \\
\midrule
Optimal & $-$ & $-$ & $0.984$ & $0.016$ & $-$ \\
\bottomrule
\end{tabular}
\end{table}

\begin{figure*}[htbp]
\centering
\includegraphics[width=0.90\textwidth]{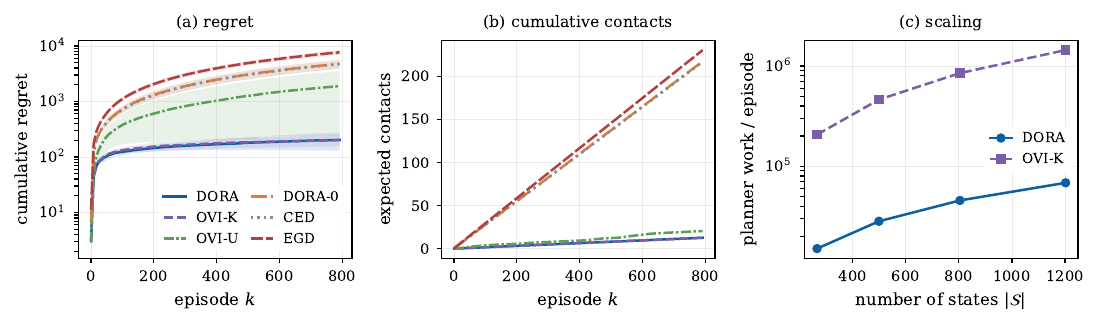}
\caption{Online warehouse navigation results. (a) Cumulative regret, shaded by one standard
deviation over ten instances. (b) Expected contacts accumulated during
learning. (c) Planner work per episode as the map grows.}
\label{fig:exp1}
\end{figure*}

First, we observe that DORA attains a regret of $203$, which is
statistically indistinguishable from the $203$ attained by OVI-K, while
performing $10.9$ times less planner work. It also beats OVI-U, which must
estimate the transition kernel, by a factor of $9.4$ in regret and by $13.0$ in
work. Second, the reduced cost correction is what makes this possible. The
ablation DORA-0 incurs $24$ times more regret and $17$ times more contacts,
and it converges to a policy whose cost exceeds the optimum by $5.8$ units. The
difference is that a determinized planner weighted by expected
one step cost cannot see where a slip takes the robot, and in this map a slip
into the picking zone is what causes contact. Third, safety during learning
follows the same pattern. DORA accumulates $12.7$ expected contacts across the
whole training run, which matches the $12.7$ accumulated by an agent that
knows the true dynamics, whereas the uncorrected planners accumulate more than
$218$.

Panel (c) of Figure~\ref{fig:exp1} reports planner work as the map is enlarged
from $266$ to $1202$ states. The ratio of work between OVI-K and DORA grows
from $13.7$ to $21.1$, further supporting our argument in
Section~\ref{sec:method}. DORA performs a fixed number of oracle calls whatever
the size of the map, whereas the number of value iteration sweeps needed to
converge grows with the diameter.

\subsection{Losslessness of Dijkstra-class Algorithms}\label{sec:exp2}

The second experiment tests our theory. We sweep the slip probability from
$0$ to $0.45$ on $25$ instances with a strongly varying terrain field and no
picking zone, so that the only source of difficulty is actuation noise. For
each instance we compute the exact optimal value, evaluate the two exactness
conditions, and evaluate three members of the Dijkstra class. Results are summarized in Figure~\ref{fig:exp2}.

\begin{figure*}[htbp]
\centering
\includegraphics[width=0.90\textwidth]{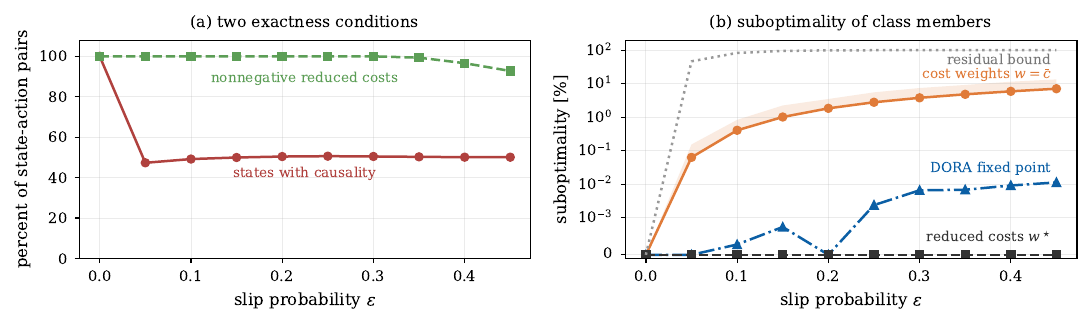}
\caption{Exactness of the Dijkstra policy class. (a) Causality collapses as
soon as slip is introduced, while the reduced costs stay nonnegative. (b) The
reduced cost weights are exactly optimal at every slip level, and the DORA
fixed point stays within $0.011$ percent of the optimum.}
\label{fig:exp2}
\end{figure*}

The classical causality condition collapses immediately. At $\varepsilon = 0$
it holds at every state, and at $\varepsilon = 0.05$ it holds at only $47.3$
percent of them. The reduced cost condition holds
at every state action pair up to a slip of $0.30$, and even at $\varepsilon =
0.45$ it holds at $92.8$ percent of them. Correspondingly, the member of the
class defined by the reduced costs is exactly optimal at every slip level
tested, to the precision of the computation. The DORA fixed point, which
reaches the same weights from data stays within
$0.011$ percent of the optimum.

Finally, we note the contrast with the two standard reference points. The member
of the class weighted by expected one step cost, which is what determinize and
replan uses, loses up to $7.1$ percent, with a ninetieth percentile loss of
$14.5$ percent. The a posteriori residual bound of \cite{hansen2011}, in which
the Bellman residual is amplified by the expected time to absorption, exceeds
the optimal value itself for any slip above $0.15$. The gap between a vacuous
worst case bound and an exact policy is the practical content of
Proposition~\ref{prop:exact}.

\subsection{Chance Constrained Navigation}\label{sec:exp3}

The third experiment enforces a risk budget. We set the contact probability in
the picking zone to $0.16$ and sweep the budget $\Delta$ from $0.30$ down to
$0.02$. DORA-S runs Algorithm~\ref{alg:dora} with the risk weight
\eqref{eq:risk} and the dual update \eqref{eq:dual}. As a reference we compute,
on the true model, the best stationary policy that meets the budget, by
bisecting a Lagrange multiplier on the true contact probability. Results are summarized in Figure~\ref{fig:exp3}.

\begin{figure*}[htbp]
\centering
\includegraphics[width=0.90\textwidth]{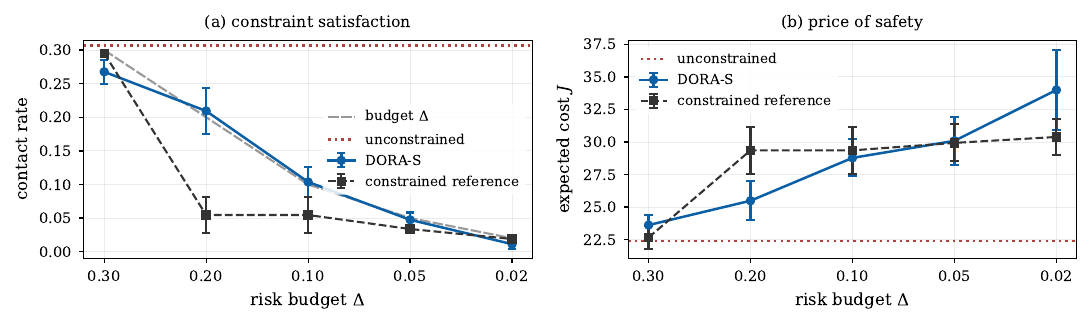}
\caption{Chance constrained navigation. (a) The contact rate of DORA-S tracks
the budget across two orders of magnitude. (b) The cost paid for that safety.
At intermediate budgets DORA-S beats the best deterministic policy because
dual ascent mixes routes across episodes.}
\label{fig:exp3}
\end{figure*}

The mean contact rate tracks the budget across the whole range, from $0.268$ at
a budget of $0.30$ down to $0.011$ at a budget of $0.02$, against an
unconstrained rate of $0.307$. The cost paid for that safety rises from $22.41$
to $33.99$, which is the price of routing around the picking zone rather than
through it. In addition, we also note two caveats. First, satisfaction is in the time
averaged sense. At the intermediate budgets the mean satisfies the constraint
but individual instances could violate it, which is expected because dual ascent
oscillates around the boundary of the feasible set. Second, at budgets of
$0.20$ and $0.10$ DORA-S attains a lower expected
cost than the best single deterministic policy that meets the budget, namely
$25.49$ against $29.35$. The reason is that the feasible set of deterministic
policies is discrete in this map. Since there is no deterministic route with a
contact rate near $0.15$, the reference must jump to a much safer and much
longer route. Note also dual ascent mixes two routes across episodes and lands
between them instead. This is the mixture structure that is known to characterize
optimal policies of constrained Markov decision processes \cite{altman1999}.

\subsection{Benchmark Problems Across Navigation Tasks}\label{sec:exp4}

The final experiment tests whether these findings are general across environments. We convert three navigation problems from the JuliaPOMDP ecosystem into
stochastic shortest path form and repeat the exactness analysis and the online
comparison on each of them. First, SimpleGridWorld, is a
$10 \times 10$ grid with actuation slip of $0.3$, one goal cell, and two
hazard cells that absorb with the dead end penalty. Second, GeoSteeringMDP from
GeoSteerings.jl \cite{geosteerings} is a directional drilling problem in which
the agent must keep the wellbore inside a sinusoidal target zone under a drift
probability of $0.3$ and reach the terminal zone at the right edge of the map.
Finally, the drone surveillance problem \cite{dronesurv} asks a UAV to travel from one
corner of a grid to the opposite corner while a ground agent performs a random
walk, and sharing a cell with the agent ends the episode. Here, traversal costs are one per move in the grid world and the drone problem. In
the geosteering problem a move that lands in the target zone costs one and a
move that lands in the shale margin costs eleven, which preserves the reward
ratio of the original model. The learner observes each step cost with uniform
noise of half width $0.2$, and every other quantity is treated exactly as in
the warehouse study. We run $K = 400$ episodes over five seeds per domain with
$\beta = 0.05$. Results appear in Figure~\ref{fig:exp4} and \ref{tab:exp5}.

\begin{figure*}[htbp]
\centering
\includegraphics[width=0.90\textwidth]{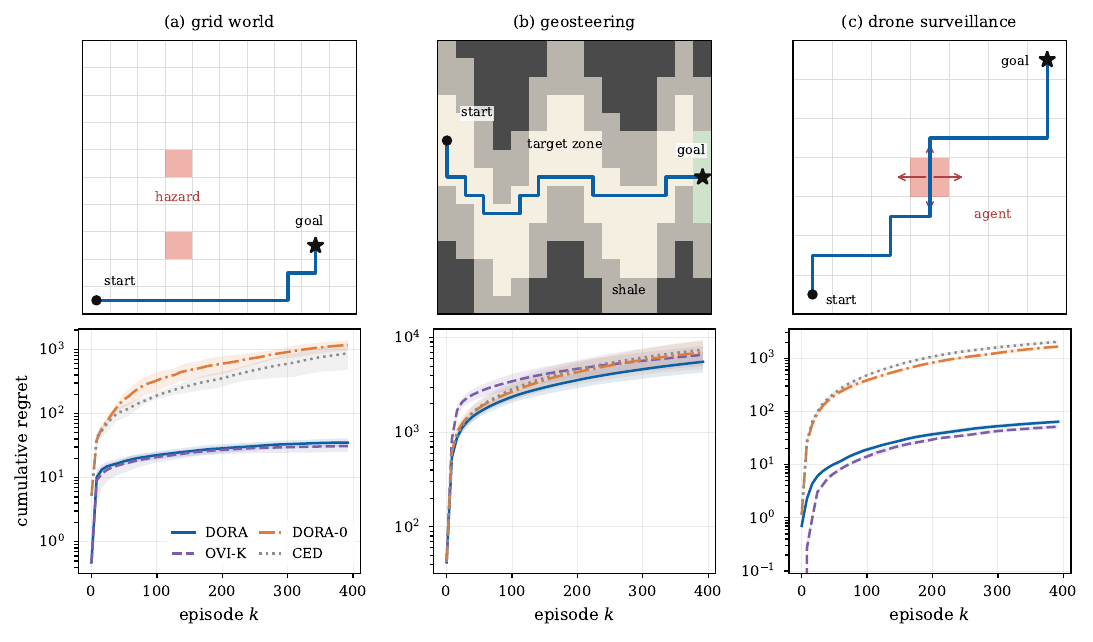}
\caption{Benchmark problems on navigation tasks. The top row shows
each problem and the optimal determinized route. Hazard cells and the ground
agent are red. In the geosteering problem the target zone is white, the shale
margin is gray, and the terminal zone is green. The bottom row shows
cumulative regret, shaded by one standard deviation over five seeds.}
\label{fig:exp4}
\end{figure*}

\begin{table}[htbp]
\centering
\caption{Online navigation metric comparison on the three navigation benchmarks tasks over $400$ episodes}
\label{tab:exp5}
\scriptsize
\setlength{\tabcolsep}{2.4pt}
\begin{tabular}{llrrrrr}
\toprule
Domain & Method & Regret $\downarrow$ & Work $\downarrow$ & SR $\uparrow$ &
CR $\downarrow$ & Contacts $\downarrow$ \\
\midrule
Grid       & DORA   & $\mathbf{35 \pm 5}$ & $5{,}916$ & $\mathbf{0.939}$ & $\mathbf{0.061}$ & $\mathbf{25.1}$ \\
world      & DORA-0 & $1{,}187 \pm 222$ & $\mathbf{388}$ & $0.874$ & $0.126$ & $57.2$ \\
           & CED    & $891 \pm 388$  & $\mathbf{388}$ & $0.883$ & $0.117$ & $43.2$ \\
           & OVI-K  & $31 \pm 6$     & $31{,}166$ & $\mathbf{0.939}$ & $\mathbf{0.061}$ & $24.9$ \\
\midrule
Geosteering & DORA   & $\mathbf{5{,}625 \pm 1{,}318}$ & $5{,}634$ & $\mathbf{1.000}$ & $0.000$ & $0.0$ \\
            & DORA-0 & $7{,}095 \pm 2{,}358$ & $\mathbf{465}$ & $\mathbf{1.000}$ & $0.000$ & $0.0$ \\
            & CED    & $7{,}577 \pm 1{,}965$ & $\mathbf{465}$ & $0.999$ & $0.000$ & $0.0$ \\
            & OVI-K  & $6{,}666 \pm 1{,}615$ & $108{,}514$ & $0.999$ & $0.000$ & $0.0$ \\
\midrule
Drone surv. & DORA   & $\mathbf{65 \pm 5}$ & $195{,}366$ & $\mathbf{1.000}$ & $\mathbf{0.000}$ & $\mathbf{0.0}$ \\
            & DORA-0 & $1{,}673 \pm 135$ & $\mathbf{9{,}847}$ & $0.803$ & $0.197$ & $83.9$ \\
            & CED    & $2{,}067 \pm 151$ & $\mathbf{9{,}847}$ & $0.781$ & $0.219$ & $88.9$ \\
            & OVI-K  & $52 \pm 5$     & $880{,}116$ & $\mathbf{1.000}$ & $\mathbf{0.000}$ & $\mathbf{0.0}$ \\
\bottomrule
\end{tabular}
\end{table}

Figure~\ref{fig:exp4} shows that the exactness picture from the warehouse
carries over. The causality condition varies widely, from $3.1$ percent of
states in the grid world to $94.1$ percent in the drone problem, while the
reduced costs stay nonnegative on at least $90$ percent of
the state action pairs in every domain. The reduced cost member of the class
is exactly optimal in the grid world and in the geosteering problem. The gap
of the member weighted by expected one step cost reaches $31.7$ percent in the
grid world, so the correction closes a larger gap on this benchmark than on
the warehouse. The drone problem is the first domain in which the clipped
reduced cost member itself loses value, namely $5.8$ percent. The negative
entries are caused by the motion of the ground agent rather than by actuation
slip. DORA nevertheless reaches a final policy within one percent of the
optimum, which shows that the class contains better members than the clipped
reduced cost weighting and that the learner finds one of them.

Table~\ref{tab:exp5} shows that the online behavior also carries over. DORA
attains a regret within one standard deviation of OVI-K on the grid world and
the drone problem, and a lower regret than OVI-K on the geosteering domain,
while performing $5.3$, $19.3$ and $4.5$ times less planner work on the three
domains. The reduced cost correction remains essential. DORA-0
and CED incur $25$ to $34$ times more regret than DORA on the grid world and
the drone problem, and they roughly double the contacts accumulated in the
grid world and drive the contact rate of the final policy to about $0.2$ in
the drone problem. The geosteering domain has no crash state, so every method
is safe there, and the correction still reduces the regret by about a fifth.

\section{Discussion}\label{sec:disc}

\subsection{Role of the Reduced Cost Views}

The label-setting literature is organized around causality, and causality is a
demanding condition. Our results suggest it is the wrong condition to check
when the question is whether the Dijkstra derived policy class is expressive
enough. Causality asks whether the value decreases along every stochastic
outcome. Exact representation only asks whether it decreases along the
determinized edge, after the drift caused by the other outcomes has been
charged to that edge. Figure~\ref{fig:exp2} quantifies the difference.

This also explains a long standing empirical puzzle. Determinize and replan is
known to work far better than worst case analysis predicts
\cite{yoon2007,little2007}, and equally known to have no useful guarantee
\cite{pineda2019}. Our interpretation is that the practical success comes from the
policy class, which is rich enough to contain the optimum, while the missing
guarantee comes from the weighting.

\subsection{Compute and Deployment}

DORA performs a fixed
number of oracle calls per episode, which does not depend on the size
of the map. Planning time is therefore more predictable, which is what a real time
control loop requires, and the oracle is the same shortest path routine that
navigation stacks already contain. Using that perspective, DORA can be added to an existing
planner as a change to the edge weights instead of a replacement of the
planning layer.

\subsection{Safety by Construction}

One observation deserves emphasis for safety critical deployment. The
determinized graph contains no edge into the crash state, so the oracle cannot
return a route that plans to collide, and risk enters only through the weights.
Optimistic value iteration has no such protection. In an early version of our
study it actively sought the absorbing failure state, because an unvisited
action looked cheap and the failure state carried no future cost. We removed
that pathology by giving every method the known dead end penalty, which is the
standard treatment in this literature \cite{kolobov2011}. Constraining the search space of the planner is a more robust safety
mechanism than relying on a value function to have learned that failure is
expensive.

\subsection{Limitations}

Finally, we highlight three limitations. First, DORA uses the actuation
model, namely the scalar slip probability and the map geometry. This is far
less than a transition kernel and is a calibrated quantity on a real platform,
but it is still a required input. An extension that estimates the slip probability online
was omitted in this study. Second, the regret guarantee
proved in Appendix~\ref{app:proofs} covers the variant whose inner iteration
is run to convergence, and the implemented algorithm with its finite inner
iteration is supported only by experiments. Third, the benchmarks are
grid maps with a small number of motion primitives. Continuous state, non
holonomic dynamics and partial observability would each require additional
work, although the reduced cost identity itself does not depend on the grid
structure. We aim to study these further extensions in future work.

\section{Conclusion}\label{sec:concl}

This paper revisits the use of Dijkstra's algorithm for SSP navigation. We showed that the policy class induced by shortest path
search contains an optimal policy whenever a reduced cost defined on the
determinized map is nonnegative, and that this condition is much weaker than
the causality condition normally invoked in past studies. On a warehouse navigation benchmark
the condition holds on every state action pair up to a slip probability of
$0.30$, and the reduced cost member of the class is exactly optimal at every
slip level we tested. Building on this, we introduced DORA (Dijkstra Oracle Reduced-cost Algorithm), an online learner
that reaches this member using a fixed number of oracle calls per episode and
no transition kernel. DORA matches optimistic value iteration that is given the
true kernel while performing an order of magnitude less planner work, and the
gap widens with the size of the map, which we validate further on three
navigation benchmarks. Future work will
estimate the actuation model online, complete the regret analysis, and extend
the reduced cost construction to continuous state planners.

\bibliographystyle{ieeetr}
\bibliography{references}

\newpage
\appendices

\section{Formal Statements and Proofs}\label{app:proofs}

This appendix expands the proofs of Propositions 1 to 3 and states and proves
the regret guarantee for the variant of Algorithm~\ref{alg:dora} whose inner
iteration is run to convergence. Throughout, $V^\star$ and $Q^\star$ denote
the exact stochastic shortest path values of $M$, which exist and satisfy the
Bellman equation under Assumption~\ref{as:reach} and $c_{\min} > 0$ by the
standard theory \cite{bertsekas1991}. The goal and the crash state are
absorbing, with terminal cost zero at the goal and $c_{\mathrm{de}}$ at the
crash state. For a weight vector $w \ge 0$ on the determinized edge set $E$,
we write $D_w(s)$ for the shortest path distance from $s$ to $g$ in the
determinized graph, which is what a backward Dijkstra call returns.

The proofs use one condition beyond the hypotheses stated in the main text.

\begin{assumption}[Strict determinized progress]\label{as:progress}
For every $s \notin \{g, \mathrm{crash}\}$,
$V^\star(s) > V^\star(\sig(s, \pi^\star(s)))$.
\end{assumption}

Assumption~\ref{as:progress} states that the intended successor of the
optimal action strictly decreases the optimal value. It is the strict form of
the causality condition restricted to the determinized edge, so it is still
much weaker than consistent improvement, which constrains every stochastic
outcome. It can fail only through an exact tie, and it holds with positive
margin on every instance in our experiments.

\begin{lemma}[Path lower bound]\label{lem:lower}
Let $s = s_1 \to s_2 \to \dots \to s_{m+1} = g$ be any path in the
determinized graph, with $s_{i+1} = \sig(s_i, a_i)$. Then
$\sum_{i=1}^{m} w^\star(s_i, a_i) \ge V^\star(s)$.
\end{lemma}

\begin{proof}
By Definition~\ref{def:rc} and $Q^\star \ge V^\star$,
\begin{align*}
\sum_{i=1}^{m} w^\star(s_i, a_i)
&= \sum_{i=1}^{m} \big( Q^\star(s_i, a_i) - V^\star(s_{i+1}) \big) \\
&\ge \sum_{i=1}^{m} \big( V^\star(s_i) - V^\star(s_{i+1}) \big),
\end{align*}
and the last sum telescopes to $V^\star(s) - V^\star(g) = V^\star(s)$.
\end{proof}

\begin{lemma}[Achievability]\label{lem:upper}
Under Assumption~\ref{as:progress}, for every $s$ the path that follows
$\pi^\star$ through the determinized graph reaches $g$, and its total weight
equals $V^\star(s)$. Hence $D_{w^\star}(s) \le V^\star(s)$.
\end{lemma}

\begin{proof}
Along the path $s_{i+1} = \sig(s_i, \pi^\star(s_i))$ the optimal value
strictly decreases by Assumption~\ref{as:progress}. Since the state space is
finite and the crash state cannot be a determinized successor, the path
cannot revisit a state and must terminate at $g$. At the optimal action
$Q^\star(s_i, \pi^\star(s_i)) = V^\star(s_i)$, so every inequality in the
proof of Lemma~\ref{lem:lower} holds with equality, and the total weight
telescopes exactly to $V^\star(s)$.
\end{proof}

\emph{Proof of Proposition~\ref{prop:exact}.} Under the hypothesis
$w^\star \ge 0$ on $E$, Dijkstra computes $D_{w^\star}$. Lemmas
\ref{lem:lower} and \ref{lem:upper} give $D_{w^\star} = V^\star$. The greedy
policy of the shortest path tree satisfies
\begin{align*}
\pi_{w^\star}(s) &\in \arg\min_a \{ w^\star(s,a) + D_{w^\star}(\sig(s,a)) \} \\
&= \arg\min_a Q^\star(s,a),
\end{align*}
so $\pi_{w^\star}$ is greedy with respect to $V^\star$ and therefore optimal
\cite{bertsekas1991}. \hfill$\blacksquare$

\emph{Proof of Proposition~\ref{prop:caus}.} At the optimal action
$Q^\star(s, \pi^\star(s)) = V^\star(s)$, and substituting into
\eqref{eq:rc} gives $w^\star(s, \pi^\star(s)) = V^\star(s) -
V^\star(\sig(s, \pi^\star(s)))$. Consistent improvement requires
$V^\star(s) > V^\star(s')$ for every successor $s'$ with positive
probability, and the determinized successor is one of them, so consistent
improvement implies $w^\star(s, \pi^\star(s)) \ge 0$. For the converse,
a lateral slip outcome may satisfy $V^\star(s') > V^\star(s)$ and violate
consistent improvement while leaving the determinized edge, and hence the
reduced cost, unchanged. Figure~\ref{fig:exp2} shows that this is the typical
case rather than the exception. \hfill$\blacksquare$

\emph{Proof of Proposition~\ref{prop:fp}.} With $\hat c = \bc$,
$\hat P = P$ and $b \equiv 0$, substituting $d = V^\star$ into \eqref{eq:w}
and comparing with \eqref{eq:drift} gives $w = [w^\star]^{+}$, which equals
$w^\star$ under the hypothesis $w^\star \ge 0$. The oracle then returns
$d^{+} = D_{w^\star} = V^\star$ by Lemmas \ref{lem:lower} and
\ref{lem:upper}, so the damped update leaves $d$ unchanged, and the returned
policy is $\pi_{w^\star}$, which is optimal by
Proposition~\ref{prop:exact}. \hfill$\blacksquare$

\subsection{Regret of the Idealized Variant}

We analyze the variant of Algorithm~\ref{alg:dora} in which the inner
iteration is run to its fixed point at every episode, so that by
Proposition~\ref{prop:fp} the executed policy $\pi_k$ is an optimal policy of
the optimistic model $M_k$ that has the true kernel $P$ and the step costs
$c_k(s,a) = \max( c_{\min}, \hat c_k(s,a) - b_k(s,a) )$. Episodes are run to
absorption. The analysis uses three conditions.

\begin{assumption}\label{as:regret}
(i) Observed step costs are $\bc(s,a) + \eta$ with $\eta$ independent, mean
zero, and bounded in $[-\bar\eta, \bar\eta]$, and the radius scale satisfies
$\beta \ge (c_{\max} + \bar\eta)/\sqrt{2}$. (ii) The exactness conditions of
Propositions \ref{prop:exact} and \ref{prop:fp} and
Assumption~\ref{as:progress} hold for every optimistic model $M_k$. (iii)
Every executed policy is proper, and its expected time to absorption from
every state is at most $\tau_{\max}$.
\end{assumption}

\begin{proposition}[Regret]\label{prop:regret}
Under Assumptions \ref{as:reach} to \ref{as:regret}, with probability at
least $1 - \delta$ the idealized variant satisfies
\[
R_K \;\le\; 4 \beta \sqrt{ \log\!\big( 2 |\Sset| |\Aset| K / \delta \big) }
\; \sqrt{ |\Sset| |\Aset| \, T_K },
\]
where $T_K$ is the total number of steps taken in the $K$ episodes. Since
$\E[T_K] \le \tau_{\max} K$, the expected regret is of order
$\tilde{O}( \beta \sqrt{ \tau_{\max} |\Sset| |\Aset| K } )$.
\end{proposition}

\begin{proof}
Let $L = \log(2|\Sset||\Aset|K/\delta)$ and let $\mathcal{G}$ be the event
that $| \hat c_k(s,a) - \bc(s,a) | \le b_k(s,a)$ for all $(s,a)$ and all
$k \le K$. Each observed cost lies in an interval of length at most
$c_{\max} + 2\bar\eta \le 2\sqrt{2}\,\beta$, so by Hoeffding's inequality and
a union bound over states, actions, episodes, and the two tail directions,
$\Prb(\mathcal{G}) \ge 1 - \delta$ with the radius \eqref{eq:bonus}. The
remainder of the proof conditions on $\mathcal{G}$.

First, optimism. On $\mathcal{G}$, $c_k(s,a) \le \bc(s,a)$ for every pair,
because $\hat c_k - b_k \le \bc$ and $c_{\min} \le \bc$. Values of a fixed
proper policy are monotone in the step costs, so
$V^\star_{M_k}(s_0) \le J^{\pi^\star}_{M_k}(s_0) \le
J^{\pi^\star}_{M}(s_0) = V^\star(s_0)$.

Second, the value difference. The models $M$ and $M_k$ share the kernel and
the terminal costs and differ only in the step costs, so for the proper
policy $\pi_k$,
\begin{align*}
J^{\pi_k}_M(s_0) - J^{\pi_k}_{M_k}(s_0)
&= \E^{\pi_k} \Big[ \sum_{t} \big( \bc - c_k \big)(s_t, a_t) \Big] \\
&\le \E^{\pi_k} \Big[ \sum_{t} 2\, b_k(s_t, a_t) \Big],
\end{align*}
where the sum runs until absorption and the inequality uses
$\bc - c_k \le \bc - (\hat c_k - b_k) \le 2 b_k$ on $\mathcal{G}$. Since
$\pi_k$ is optimal for $M_k$ by Proposition~\ref{prop:fp} and
Assumption~\ref{as:regret}(ii), $J^{\pi_k}_{M_k}(s_0) = V^\star_{M_k}(s_0)
\le V^\star(s_0)$, and the per episode regret is bounded by the displayed
expectation.

Third, the pigeonhole step. Summing the realized radii over the whole run
and letting $N_T(s,a)$ be the final visit counts,
\begin{align*}
\sum_{k=1}^{K} \sum_{t} b_k(s_t, a_t)
&\le \beta \sqrt{L} \sum_{s,a} \sum_{j=1}^{N_T(s,a)} \frac{1}{\sqrt{j}} \\
&\le 2 \beta \sqrt{L} \sum_{s,a} \sqrt{N_T(s,a)},
\end{align*}
and by the Cauchy Schwarz inequality the last sum is at most
$\sqrt{ |\Sset||\Aset| \, T_K }$. Combining the three steps gives the claim.
\end{proof}

Proposition~\ref{prop:regret} covers the idealized variant. The implemented
algorithm differs in two ways: the inner loop performs $I$ damped iterations
rather than running to convergence, and the clipping in \eqref{eq:w} can be
active when the optimistic reduced costs are negative.
Proposition~\ref{prop:fp} identifies the fixed point that the inner loop
targets, and Section~\ref{sec:exp} verifies empirically that $I = 3$ damped
iterations track the idealized behavior. Closing this gap formally is left to
future work, as stated in Section~\ref{sec:disc}.

\section{Experimental Parameters}\label{app:params}

Tables \ref{tab:params-wh} and \ref{tab:params-bench} list every parameter of
the four experiments. The warehouse model and the learners are described in
Sections \ref{sec:problem} to \ref{sec:exp}, and the benchmark conversions in
Section~\ref{sec:exp4}.

\begin{table}[t]
\centering
\caption{Parameters of the warehouse experiments.}
\label{tab:params-wh}
\footnotesize
\setlength{\tabcolsep}{3.0pt}
\begin{tabular}{llll}
\toprule
Parameter & Exp.\ 1 & Exp.\ 2 & Exp.\ 3 \\
\midrule
Map size & $20 \times 20$ & $20 \times 20$ & $20 \times 20$ \\
States $|\Sset|$ & $266$ & $266$ & $266$ \\
Slip probability $\varepsilon$ & $0.10$ & $0$ to $0.45$ & $0.10$ \\
Picking zone contact prob. & $0.14$ & $0$ & $0.16$ \\
Terrain roughness & $0.5$ & $1.6$ & $0.5$ \\
Cost bounds $[c_{\min}, c_{\max}]$ & \multicolumn{3}{c}{$[0.25,\, 2.2]$} \\
Bump cost & $1.5$ & $1.5$ & $1.5$ \\
Dead end penalty $c_{\mathrm{de}}$ & $80$ & $80$ & $25$ \\
Timeout penalty $c_{\mathrm{to}}$ & $60$ & $60$ & $60$ \\
Horizon $H$ & $150$ & $150$ & $150$ \\
Cost noise half width & $0.12$ & $-$ & $0.12$ \\
Episodes $K$ & $800$ & $-$ & $800$ \\
Instances & $10$ & $25$ & $10$ \\
Radius scale $\beta$, confidence $\delta$ & $0.05$, $0.1$ & $-$ & $0.05$, $0.1$ \\
Inner iterations $I$, damping $\alpha$ & $3$, $0.4$ & $8$, $0.4$ & $3$, $0.4$ \\
OVI sweep cap, tolerance & $200$, $10^{-4}$ & $-$ & $-$ \\
EGD exploration rate & $0.10$ & $-$ & $-$ \\
Scaling sizes & $\{20,28,36,44\}$ & $-$ & $-$ \\
Scaling episodes, seeds & $60$, $3$ & $-$ & $-$ \\
Dual step $\eta_0$, cap $\lambda_{\max}$ & $-$ & $-$ & $300$, $300$ \\
\bottomrule
\end{tabular}
\end{table}

\begin{table}[t]
\centering
\caption{Parameters of the benchmark domains of experiment 4.}
\label{tab:params-bench}
\footnotesize
\setlength{\tabcolsep}{2.6pt}
\begin{tabular}{llll}
\toprule
Parameter & Grid world & Geosteering & Drone surv. \\
\midrule
Grid size & $10 \times 10$ & $15 \times 15$ & $7 \times 7$ \\
States $|\Sset|$ & $99$ & $157$ & $2211$ \\
Actions & $4$ & $3$ & $5$ \\
Motion noise & slip $0.3$ & drift $0.3$ & agent walk \\
Start & $(1,1)$ & $(1,10)$ & $(1,1)$, agent $(4,4)$ \\
Goal & $(9,3)$ & right edge & $(7,7)$ \\
Hazards & $(4,3)$, $(4,6)$ & shale margin & meeting agent \\
Step cost & $1$ & $1$ / $11$ & $1$ \\
Dead end penalty $c_{\mathrm{de}}$ & $20$ & $60$ & $25$ \\
Timeout penalty $c_{\mathrm{to}}$ & $30$ & $120$ & $40$ \\
Horizon $H$ & $100$ & $100$ & $100$ \\
Cost noise half width, $c_{\min}$ & $0.2$, $0.5$ & $0.2$, $0.5$ & $0.2$, $0.5$ \\
Episodes $K$, seeds & $400$, $5$ & $400$, $5$ & $400$, $5$ \\
Map generator & default & default & restricted agent \\
\bottomrule
\end{tabular}
\end{table}

The geosteering map uses the default generator of the original package, with
base amplitude $3.0$, base frequency $1.0$, amplitude variation $0.5$,
frequency variation $0.05$, phase $0.3$, vertical shift $7.0$, and target
thickness $5.0$. All learners in every experiment use the radius scale
$\beta = 0.05$ and the confidence parameter $\delta = 0.1$.

\section{Policy Rollouts}\label{app:rollouts}

Figures \ref{fig:roll-wh} to \ref{fig:roll-ds} compare one rollout of the
final DORA policy with one rollout of the final OVI-K policy in each domain. The sequences illustrate the main quantitative finding of
Section~\ref{sec:exp}. In every domain the two final policies follow nearly
the same route, although DORA computed its policy with a fixed number of
shortest path calls per episode while OVI-K ran value iteration with the true
kernel. In the warehouse both policies route above the picking zone, which is
optimal at this dead end penalty. In the grid world both detour below the
hazard cells, in the geosteering domain both follow the sinusoidal target
zone and absorb the same drift events, and in the drone domain both curve
around the region that the ground agent occupies.

\begin{figure*}[htbp]
\centering
\includegraphics[width=0.92\textwidth]{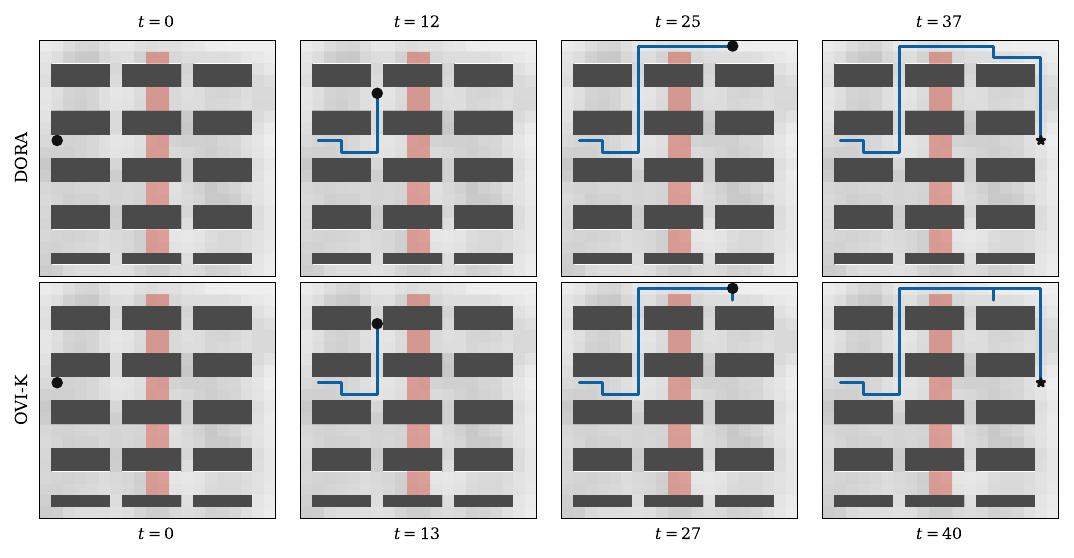}
\caption{Rollout of the final DORA policy (top) and the final OVI-K policy
(bottom) on the warehouse instance of experiment 1.}
\label{fig:roll-wh}
\end{figure*}

\begin{figure*}[htbp]
\centering
\includegraphics[width=0.92\textwidth]{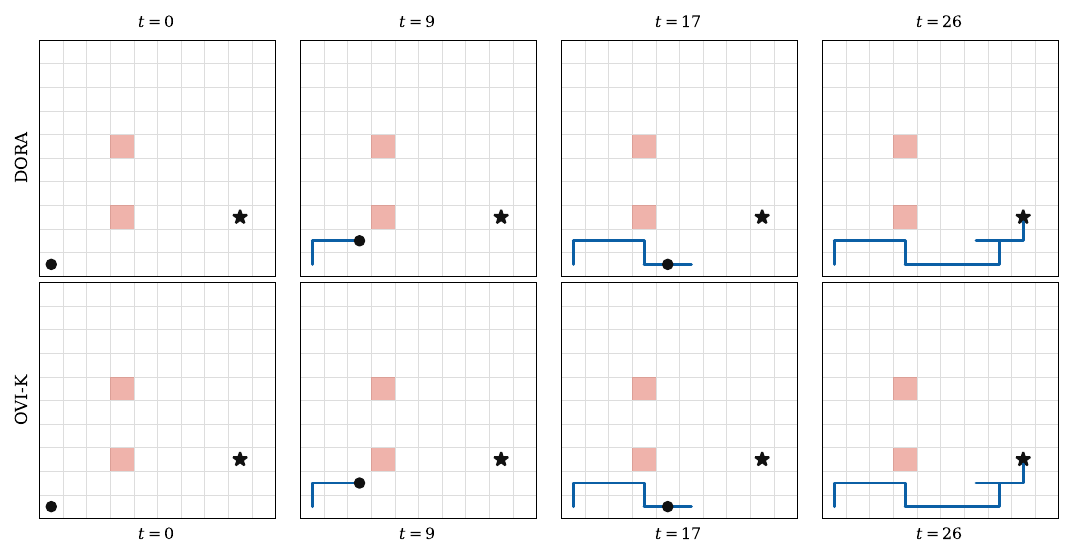}
\caption{Rollout of the final policies on the grid world benchmark. Red
cells are the hazard cells.}
\label{fig:roll-gw}
\end{figure*}

\begin{figure*}[htbp]
\centering
\includegraphics[width=0.92\textwidth]{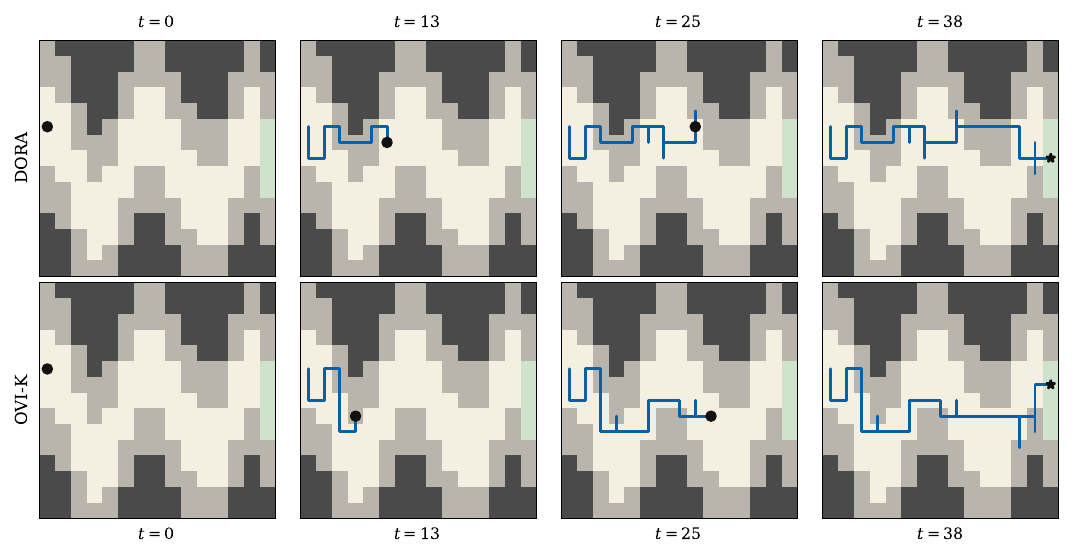}
\caption{Rollout of the final policies on the geosteering benchmark. The
target zone is white, the shale margin is gray, and the terminal zone is
green.}
\label{fig:roll-gs}
\end{figure*}

\begin{figure*}[htbp]
\centering
\includegraphics[width=0.92\textwidth]{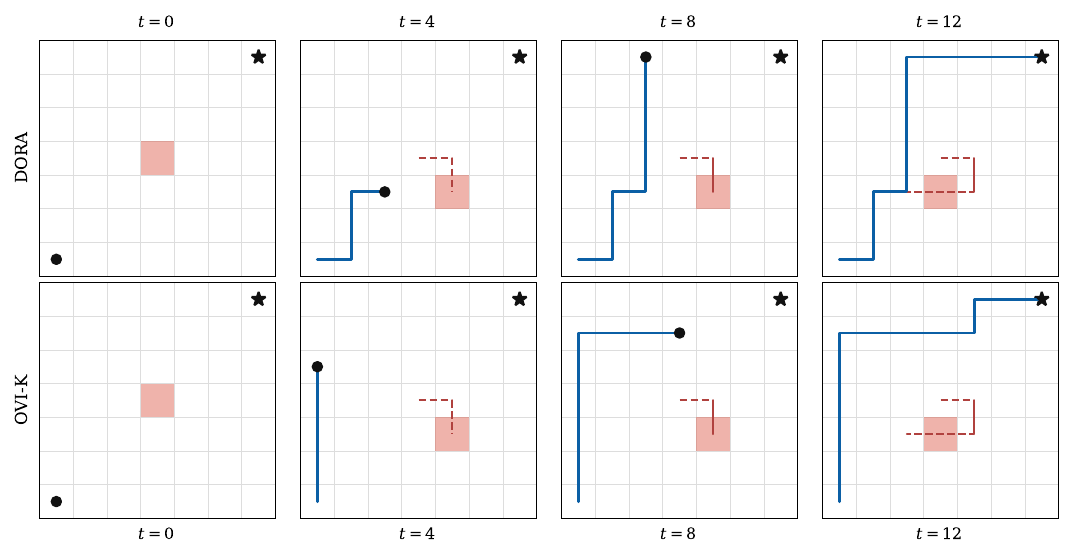}
\caption{Rollout of the final policies on the drone surveillance benchmark.
The dashed red line is the path of the ground agent.}
\label{fig:roll-ds}
\end{figure*}

\end{document}